\documentclass[letterpaper, 10 pt, conference]{ieeeconf}  % Comment this line out
\IEEEoverridecommandlockouts                              % This command is only
\usepackage{amsfonts}
\usepackage{eucal}
\usepackage{amsmath}
\usepackage{cite}
\usepackage{graphicx}
\usepackage{multirow}
\usepackage{array}
\usepackage{url}
\newtheorem{thm}{Theorem}

\newtheorem{algorithm}{Algorithm}
\newtheorem{problem}{Problem}

\newcommand{\LTLUNTIL}{\ensuremath{\ \mathcal{U}\ }}
\newcommand{\LTLNEXT}{\ensuremath{\ \bigcirc\ }}
\newcommand{\LTLEVENTUALLY}{\ensuremath{\ \diamondsuit\ }}
\newcommand{\ARUN}{\ensuremath{r}}

\newcommand{\tab}{\hspace*{4em}}

\begin{document}
\title{\LARGE \bf Technical Report: A Receding Horizon Algorithm for Informative Path Planning with Temporal Logic Constraints}

\author{Austin Jones, Mac Schwager, and Calin Belta}
%\and
%\IEEEauthorblockN{Mac Schwager}
%\IEEEauthorblockA{Department of Mechanical Engineering \\ Boston University \\ Boston, MA 02115 \\ schwager@bu.edu}
%\and
%\IEEEauthorblockN{Calin Belta}
%\IEEEauthorblockA{Department of Mechanical Engineering \\ Division of Systems Engineering \\ Boston University \\ Boston, MA 02115 \\ cbelta@bu.edu}}


%Telephone: (800) 555--1212, Fax: (888) 555--1212}


\maketitle
\thispagestyle{empty}
\pagestyle{empty}
\let\thefootnote\relax\footnotetext{Austin Jones is with the Division
  of Systems Engineering, Mac Schwager and Calin Belta are with the
  Division of Systems Engineering and the Department of Mechanical
  Engineering at Boston University, Boston, MA 02115. Email:
  \{austinmj,schwager,cbelta\}@bu.edu

This work was partially supported by ONR under grants MURI N00014-09-1051 and ONR MURI N00014-10-10952 and by NSF under grant CNS-1035588}

\begin{abstract}
%\boldmath
  This technical report is an extended version of the paper
  'A Receding Horizon Algorithm for Informative Path Planning with Temporal Logic Constraints'
  accepted to the 2013 IEEE International Conference on Robotics
  and Automation (ICRA).
  
  This paper considers the problem of finding the most informative
  path for a sensing robot under temporal logic constraints, a richer
  set of constraints than have previously been considered in
  information gathering.  An algorithm for informative path planning
  is presented that leverages tools from information theory and formal
  control synthesis, and is proven to give a path that satisfies the
  given temporal logic constraints.  The algorithm uses a receding
  horizon approach in order to provide a reactive, on-line solution
  while mitigating computational complexity.  Statistics compiled from
  multiple simulation studies indicate that this algorithm performs
  better than a baseline exhaustive search approach.
\end{abstract}


% \fi
\IEEEpeerreviewmaketitle

\section{Introduction}
\label{introduction}

In this paper we propose an algorithm for controlling a mobile sensing
robot to collect the most valuable information in its environment,
while simultaneously carrying out a required sequence of actions
described by a temporal logic (TL) specification.  Our algorithm is
useful in situations where a robot's main objective is to collect
information, but it must also perform pre-specified actions
for the sake of safety or reliability. Consider searching for a survivor trapped in the
rubble of a collapsed building.  Our
algorithm would drive the robot to locate the survivor while  
avoiding obstacles, and returning to a rescue worker to
report on the progress of its search.  The obstacle avoidance and
visit to the worker are represented as temporal logic
constraints.  In order to locate the survivor, the robot plans a path
on-line, in a receding horizon fashion, such that it localizes the
survivor as precisely as possible, while still satisfying the temporal
logic constraints.

This work brings together methods from information theory and formal
control synthesis to create new tools for robotic information
gathering under complex constraints.  More specifically, the robot uses
a recursive Bayesian filter to estimate the quantity of interest in
its environment (e.g. the location of a survivor) from its noisy sensor measurements.  The Shannon entropy of
the Bayesian estimate is used as a measure of the robot's uncertainty
about the quantity of interest.  The robot plans a path to
maximally decrease the expected entropy of its estimate over a finite
time horizon, subject to the TL constraints.  The path
planning is repeated at each time step as the Bayesian filter is
updated with new sensor measurements to give a reactive, receding
horizon planner.  
%Taking inspiration from recent methods in formal
%control synthesis~\cite{DiLaBe-ACC-2012}, we prove that 
Our algorithm is
guaranteed to satisfy the TL specification.  We
compare the performance of our algorithm to a non-reactive, exhaustive
search method.  We show in statistics compiled from extensive
simulations that our receding horizon algorithm gives a lower entropy
estimate  with lower computational
complexity than the exhaustive search method.

The algorithm we present is applicable to many scenarios in which we
want a robot to gather informative data, but where safety and
reliability are critical.  For example, our algorithm can be used by a
mobile robot deployed on Mars that is tasked with collecting soil
samples and images while gathering enough sunlight to charge its
batteries and avoiding dangerous terrain.  In an animal population
monitoring scenario, our algorithm can drive a robot to count animals
of a given species whose positions are unknown while avoiding
sensitive flora and fauna, eventually uploading data to scientists.
Our algorithm could also be used, for example, in active SLAM to
control a robot to build a minimum uncertainty map \cite{ThrunBook08SlaM}
of its environment while avoiding walls and returning to a base
station for charging.

Extensive work already exists in using information theoretic tools in
robotic information gathering applications.  Most of this work uses a
one-step-look-ahead approach \cite{BourgaultICRA02infAdapt,SchwagerISRR11Hazards}, a
receding horizon approach \cite{GanICRA12recHor,BinneyICRA2010infPlan}, or an offline plan
based on the sub-modularity property of mutual information
\cite{SinghJAIR09efficientinformative,SinghIJCAI07infPlan,MelioAAAAI07nstip}.  The key innovation
in our algorithm is that it gives a path which is guaranteed to
satisfy rich temporal logic constraints.  Temporal logic constraints
can specify complex, layered temporal action sequences that are
considerably more expressive than the static constraints considered in
previous works.  Indeed, much of the work in constrained informative
path planning can be phrased as a special case of the TL constraints
that we consider here.  For example the authors of \cite{BinneyICRA2010infPlan}
solve an information-gathering problem in which an underwater agent
must avoid high traffic areas and communicate with
researchers---constraints which can be naturally expressed as TL
statements.

In this work, we consider a particular kind of temporal logic called
syntactically co-safe linear temporal logic (scLTL)
\cite{VardiFormMeth01safety}.  Synthesis of trajectories from scLTL
specifications is currently an active area of research
\cite{GolHscc12scltl,BhatiaCDC2010scltl}, as is the use of receding horizon
control to solve optimization problems over TL-constrained systems.
Receding horizon control (RHC), sometimes referred to as
model-predictive control, is a control technique in which current
information is used to predict performance over a finite horizon
\cite{MayneAuto99thSurvey,QinConEng03inSurvey}.  The authors of \cite{WongpiromsarnHSCC10heir} use a receding
horizon path planning algorithm that satisfies TL constraints in a
provably correct manner and is capable of correcting navigational
errors on-line.  In \cite{DiLaBe-ACC-2012}, the authors extended this principle
to provide a receding horizon algorithm for gathering time-varying,
deterministic rewards in a TL-constrained system.  The analysis of our
informative planning algorithm was inspired by \cite{DiLaBe-ACC-2012}, with the
significant difference that information gain is a stochastic quantity
which depends on noisy sensor measurements.

The paper is outlined as follows. We define the necessary mathematical
preliminaries in Section~\ref{mathMod}.  In Section~\ref{infPathPlan},
we formalize the scLTL-constrained informative path planning problem.
In Section~\ref{algDesc} we present our receding horizon algorithm,
and prove that it satisfies the scLTL constraints.  Results from
simulations comparing our algorithm to a baseline exhaustive search
method are presented in Section~\ref{caseStudy}.  Finally, in
Section~\ref{goals}, we give our conclusions and discuss directions
for future work.

As mentioned in the abstract, this report is an extended version of a
paper accepted to the 2013 ICRA conference.  The main additions
are some information theory definitions in Section \ref{mathMod}
and a proof of Theorem \ref{convergeThm} in Section \ref{algDesc}.

\section{Notation and definitions}
\label{mathMod}
 For a set $S$, we use $|S|$ and $2^S$ to denote its cardinality and power set, respectively.  $S \times T$ is the Cartesian product of $S$ and $T$.  
 
%  We denote the range space of a discrete random variable $X$ as $R_X$, its realization as $x \in R_X$, and its probability mass function (pmf) as  $p_X$.  We denote the  \em (Shannon) entropy, conditional entropy, \em and \em mutual information \em  of discrete random variables $X,Y$ \cite{ShannonBell48ent,CoverBook06elemInfo}   as $H(X),H(X|Y),$ and $I(X;Y)$, respectively.
 
 \subsection{Information theory}
\label{infTheory}
Information theory is a general mathematical theory of communication \cite{ShannonBell48ent}.  In this work we borrow  from information theory measures of uncertainty and information content of discrete random variables.

 We denote the range space of a discrete random variable $X$ as $R_X$, its realization as $x \in R_X$, and its probability mass function (pmf) as  $p_X$.  The  \em Shannon entropy \em  \cite{ShannonBell48ent} (referred to in the sequel as simply ``entropy") of a discrete random variable $X$ is 

\begin{equation}
\label{entropy}
H(X) = H(p_X) = -\sum_{x \in R_X} p_X(x) \log(p_X(x)).
\end{equation}

The logarithm in \eqref{entropy} is base 2 by convention and entropy is measured in units of ``bits".  The entropy of a random variable is a measure of its ``randomness" or ``uncertainty".   
For a fixed range space $R_X$,  $H(X) \in [0, \log(|R_X|)]$.   A uniformly distributed random variable achieves the upper bound and a deterministic variable achieves the lower bound.  
%This matches intuition, as we can perfectly guess the value of a deterministic variable and would have to make many guesses about a realization from a uniform distribution.

For two discrete random variables $X$ and $Y$,  the \em conditional entropy \em  \cite{CoverBook06elemInfo} of $X$ given $Y$ is

\begin{align}
\label{condEnt}
\begin{split}
H(X|Y)   = & H(p_X|p_Y) \\  = &  -\sum_{y \in R_Y} \sum_{x \in R_X} p_{X,Y}(x,y) \log(p_{X|Y}(x|y)).
%\\ = & \sum_{y \in R_Y} H(X|Y=y)p_Y(y). \\
\end{split}
\end{align}

In general, $H(X|Y) \in [0,H(X)]$.  $H(X|Y)$ is a measure of how random $X$ is if we are given knowledge of $Y$ and the statistical relationship between $X$ and $Y$.

The \em mutual information \em  \cite{CoverBook06elemInfo} between two random variables is

\begin{align}
\label{mutInf}
\begin{split}
 I(X;Y) = &  I(p_X;p_Y) \\ = & \sum_{x \in R_X} \sum_{y \in R_Y} p_{X,Y}(x,y) \log(\frac{p_{X,Y}(x,y)}{p_X(x)p_Y(y)}). \\ 
% = &  H(X)-H(X|Y). \\
 \end{split}
\end{align}

The identity $I(X;Y) = H(X)-H(X|Y)$ leads to the natural interpretation of mutual information as  the increase in certainty of  $X$ when we have knowledge of $Y$.

\subsection{Transition systems and syntactically co-safe LTL}
\label{discMod}

A \em weighted transition system \em \cite{BaierText08Model}  is a tuple $TS =  (Q,q_0,Act, Trans, AP, L,d)$, where $Q$ is a set of states,
$q_0 \in Q$ is the initial state, $Act$ is a set of actions, $Trans \subseteq Q \times Act \times Q$ is a transition relation, $AP$ is a set of atomic propositions, $L: Q \to 2^{AP}$ is a labeling function of states to atomic propositions, and $d:Trans \to \mathbb{R}$ is a weighting function over the set of transitions.
%

%A finite word $a^0 \ldots a^{n-1} \in {Act}^{n}$ defines a run $q^0 \ldots q^n \in Q^{n+1}$ $(q^0=q_0)$ such that $(q^i,a^i,q^{i+1}) \in Trans$ $\forall i \in [0,n-1]$. 
%%The trace of the word $a^0 \ldots a^n $ is given as $L(q^0) \ldots L(q^{n}) \in (2^{AP})^{n+1}$. 
%

A  \em finite state automaton \em (FSA) is a tuple  $\mathcal{A} = (\Sigma, \Pi, \Sigma_0, F, \Delta_\mathcal{A})$ where
$\Sigma$ is a finite set of states, $\Pi$ is an input alphabet, $\Sigma_0 \subseteq \Sigma$ is a set of initial states, $F \subseteq \Sigma$ is a set of final (accepting) states, and $\Delta_\mathcal{A} \subseteq \Sigma \times \Pi \times \Sigma$ is a deterministic transition relation.

An \em accepting run \em $\ARUN_\mathcal{A}$ of an automaton $\mathcal{A}$ on a finite word $w=w^0w^1 \ldots w^j$ over $\Pi$ is a sequence of states $\ARUN_\mathcal{A}=\sigma^0\sigma^1 \ldots \sigma^{j+1}$ such that  $\sigma^{j+1} \in F$ and $(\sigma^i,w^i,\sigma^{i+1})  \in \Delta_\mathcal{A}$ $\forall i \in [0,j]$. 

%The set of all words corresponding to all of the  accepting  runs of $\mathcal{A}$ is called the \em language \em accepted by $\mathcal{A}$ and is denoted as $\LANG_\mathcal{A}$.
%

The \em product automaton \em between a weighted transition system $TS =  (Q,q_0,Act, Trans, AP, L,d)$ and an FSA $\mathcal{A} = (\Sigma, \Pi, \Sigma_0, F, \Delta_\mathcal{A})$  with $\Pi=2^{AP}$ is a tuple $\mathcal{P} = TS \times \mathcal{A} = (Q \times \Sigma, q_0 \times \Sigma_0,\Delta_{\mathcal{P}},Q \times F,d')$ \cite{BaierText08Model}.  The transition relation and weighting are defined as $\Delta_{\mathcal{P}} = \{(q,\sigma),\pi,(q',\sigma') |  (q,\pi,q') \in Trans, (\sigma,L(q),\sigma') \in \Delta_{\mathcal{A}}\}$ and $d'((q,\sigma),\pi,(q',\sigma')) = d(q,\pi,q')$, respectively. 

An \em accepting run \em $\ARUN_{\mathcal{P}}$ on a finite word $\pi = \pi^0\pi^1 \ldots \pi^{j}$ is a sequence of states $\ARUN_{\mathcal{P}} = (q^0,\sigma^0)(q^1,\sigma^1) \ldots (q^{j+1},\sigma^{j+1})$ such that $(q^0,\sigma^0) \in \{q_0\} \times \Sigma_0$, $(q^{j+1},\sigma^{j+1}) \in Q \times F$, and $((q^i,\sigma^i),\pi^i,(q^{i+1},\sigma^{i+1})) \in \Delta_{\mathcal{P}}$ $\forall i \in [0,j]$.  

%The collection of all words corresponding to an accepting run of $\mathcal{P}$ is called the \em language \em of $\mathcal{P}$, denoted by $\LANG_{\mathcal{P}}$.
%
The \em projection \em of a run $(q^0,\sigma^0) \ldots (q^j,\sigma^j)$ from $\mathcal{P}$ to $TS$ is the run $q^0 \ldots q^j$ over $TS$.
%
%\em Syntactically co-safe linear temporal logic \em (scLTL) is a fragment of linear temporal logic (LTL) where the satisfaction of infinite words can be checked with finite prefixes \cite{GolHscc12scltl,VardiFormMeth01safety}.
 
 \em Syntactically co-safe linear temporal logic \em formulas are made of atomic propositions along with the Boolean operators  ``conjunction" ($\wedge$), ``disjunction" ($\vee$) and ``negation" ($\neg$)  and the temporal operators ``until" ($\LTLUNTIL$), ``next" ($\LTLNEXT$), and ``eventually" ($\LTLEVENTUALLY$) \cite{VardiFormMeth01safety}. 

\section{Problem formulation}
\label{infPathPlan}
Our task is to find a path such that a robot following it fulfills a temporal logic task specification and also on average produces a low-entropy estimate of some \em a priori \em unknown quantity. We model a robot as a deterministic transition system over which we can evaluate temporal logic specifications and provide a model for incorporating new information into the robot's estimate.  We use these models to formalize the scLTL-constrained informative path planning problem.

\subsection{Robot motion model}
We consider a robot with known kinematic state moving deterministically in an environment.  Here we have taken a hierarchical view of path planning \cite{KaelbingICRA11hieirarch,WongpiromsarnHSCC10heir} in which the problem is decomposed into the  high-level problem of selecting way points  on a graph to be followed by the robot and the low-level problem of selecting local trajectories between nodes.  We assume that the low-level problem is solved and focus on high-level path planning.  We partition the environment and take the quotient to form a transition system $TS = (Q,q_0,Act, Trans, AP, L,d)$ \cite{BaierText08Model}, where $Q$ is the set of regions in the partition and $q_0 \in Q$ is the region where the robot is located initially.  $Act$ is a set of finite-time control policies  $\{u_i\}_{i \in [1,|Act|]}$ that can be enacted by the robot. A transition $(q_i,u_k,q_j) \in Trans$ is a pair of regions $q_i$ and $q_j$ and the control policy $u_k$ that can be applied to drive the robot from $q_i$ to $q_j$. 
$AP$ is a set of properties that can be assigned to regions in $Q$ and $L:Q \to 2^{AP}$ is the mapping giving the set of properties satisfied at each region. $d: Trans \to \mathbb{R}$ is a weighting over the transitions whose value corresponds to the cost of enacting the given control.  We define the discretized time $k$ that is initialized to 0 and incremented by 1 after a transition.  We denote the state of $TS$ at a time $k$ as $q^k$.

\subsection{Estimator and sensor dynamics}
\label{sensDyn}
The robot is tasked with estimating an environmental feature modeled as the random variable $S$.  We assume that the robot has onboard sensors and can take and process measurements related to $S$.  We  encapsulate the measurement and data-processing performed during a transition on $TS$ at time $k$ as a report $y^k$.  The report is drawn from a random process $Y^k$ whose randomness encapsulates sensor noise. 
% Due to sensor noise, we say that $y^k$ is a realization of a random process $Y^k$. 
 The pmf of $Y^k$ depends on the realization $s$ of $S$, the position of the robot, and sensor statistics.   We can use this model to construct a likelihood function $f(y^k,s,q^k) = Pr[Y^k=y^k|S=s, $ robot at $ q^k ]$. The robot maintains an estimate pmf $\hat{p}:R_S \times \mathbb{N} \to [0,1]$, where $\hat{p}(s,l) = Pr[S=s|\{Y^j =y^j\}_{j \in [0,l]}]$. After a transition is completed, the robot incorporates the report into $\hat{p}$ via a Bayes filter

\begin{equation}
\label{bayesFilter}
\hat{p}(s,l+1) = \frac{f(y^{l+1},s,q^{l+1})\hat{p}(s,l)}{\sum_{\sigma \in R_S} f(y^{l+1},\sigma ,q^{l+1})\hat{p}(\sigma,l)}.
\end{equation}

%  Define the function $\psi(y^k,s,q^k) = Pr[Y^k=y^k| S=s, $ robot at $ q^k $, no overlap$]$ and   Then we can construct a mapping $\Psi$ to incorporate our model of information bleed into the likelihood function $f$.
%
%\begin{align}
%\begin{split}
%\Psi(\{(d_M(q^l,q_k)\},\{\psi(y^l,s,q_k))\}_{q_k \in N_o(q^l) \cup q^l})  & \\ = f(y^l,s,q^l) & \\ 
%\end{split}
%\end{align}
%
%The form of $\Psi$ depends on the sensing model and on how we define the weighting $d_M$.  For instance, if $d_M$ is the Euclidean distance between the centers of regions and our sensor is a detector whose probability of correct detection decreases exponentially with 

\subsection{scLTL-constrained informative path planning}
\label{ipp}
%Abstractly, our task is to select the trajectory such that on average, enacting the trajectory will allow the robot to form the best estimate of $S$.  
Our task is to select a sequence of transitions $\{(q^i,u^i,q^{i+1})\}_{i \in [0,k-1]}$ such that the induced run $q^0 \ldots q^k$ over $TS$  on average produces the best estimate of $S$. The robot's knowledge of $S$ is given by its estimate $\hat{p}(\cdot,k)$.  We quantify the impact on $\hat{p}(\cdot,k)$ of a set of transitions  by using the mutual information $I(\hat{p}(\cdot,k);\{ Y^l \}_{l \in [0,k]})$,  a frequently-used measure of sensing quality in sensor networks, localization, and surveillance problems \cite{BourgaultICRA02infAdapt,SchwagerISRR11Hazards,KrauseISPN06optimalSense,MeguerdichianINFOCOM01optimalSense}.  Since our goal is to produce the best estimate, we  naturally wish to maximize the mutual information.  We may restate this objective by using the identity $I(\hat{p}(\cdot,k);\{Y^j\}_{j \in [0,k]}) = H(\hat{p}(\cdot,k)) - H(\hat{p}(\cdot,k)| \{Y^j\}_{j \in [0,k]} )$.   The estimate $\hat{p}(\cdot,\cdot)$ does not change over time if no new reports are received, so maximizing the mutual 
information is equivalent to minimizing the conditional entropy $H(\hat{p}(\cdot,k)| \{Y^j\}_{j \in [0,k]} )$.

\begin{problem}
\label{conippp}
The scLTL-constrained informative path planning problem over $TS$ is the optimization  

\begin{equation}
\label{def2}
\begin{array}{c}
 \underset{\{q^j\}_{j\in [0,k]}} \min E_{\{Y^j\}} [H(\hat{p}(\cdot,k)|\{Y^j\})] \\
\text{subject to} \\
\phi \\
(q^i,u^i,q^{i+1}) \in Trans \text{  } \forall i \in [0,k-1], \\ 
\end{array}
\end{equation}

\noindent
 
where $\phi$ is an scLTL formula over $AP$, the likelihood function $f$ and initial pmf $\hat{p}(\cdot,0)$ are given, and $k$ is finite but not fixed.
\end{problem}

%\begin{equation}
%\label{def1}
%\begin{array}{c}
%\max_{\{(q^m,u^m,q^{m+1})\}_{m \in [0,k-1]}} E_{\{\Omega_j\}_{j \in [0,h]}} [I(\hat{p}(\cdot,h);\{\zeta(q^j,s,j,\Omega_j)\})] \\
%\text{subject to} \\
%\phi \\
%\end{array}
%\end{equation}

%\begin{equation}
%\label{def2}
%\begin{array}{c}
%\min_{\{(q^m,u^m,q^{m+1})\}_{m \in [0,k-1]}} E_{\{\Omega_j\}_{j \in [0,k]}} [H(\hat{p}(\cdot,h)|\{\zeta(q^j,s,j,\Omega_j)\})] \\
%\text{subject to} \\
%\phi \\
%\end{array}
%\end{equation}

%This formulation gives a more natural problem interpretation than mutual information maximization.    If the conditional entropy of the estimator achieves its highest possible value $ (\log(|R_S|))$, this means our estimate is uniform, and using it to predict the value of $S$ would perform as well as using a roll of the dice. If, on the other hand, the objective achieves its lowest possible value, ($0$), then the estimate has a deterministic form. By lowering its conditional entropy, we increase how helpful knowledge of $\hat{p}(\cdot,k)$ is when trying to determine $S$. 

We discuss how to use model checking and optimization tools to solve this problem in the next section.

%%%%%%%%%%%%%%%%%%
\section{Receding horizon informative path planning}
%%%%%%%%%%%%%%%%%%
\label{algDesc}

From an scLTL formula $\phi$, we can construct an FSA $\mathcal{A}_{\phi}$ that will accept only those words that satisfy $\phi$ \cite{VardiFormMeth01safety,LatvalaBook03scheck}.  Given $TS$ and $\phi$, we can construct a product FSA $\mathcal{P} = TS \times \mathcal{A}_{\phi}$. Accepting runs over $\mathcal{P}$ are given as finite words $(q^0,\sigma^0) \ldots (q^k,\sigma^k)$ such that transitions between subsequent states are in $\Delta_{\mathcal{P}}$ and $\sigma^k \in F$. Problem \ref{conippp} can be solved using the following procedure:

\begin{algorithm}[Exhaustive Search]\
\label{brForce}

\begin{enumerate}
\item
From $TS$ and $\phi$, construct $\mathcal{P} = TS \times \mathcal{A}_{\phi}$
\item
Enumerate all  accepting runs of $\mathcal{P}$, i.e. all simple paths from $(q^0,\sigma^0)$ to states in $Q \times F$
\item
Project all accepting runs on $\mathcal{P}$ to runs over $TS$
\item
Calculate $E_{\{Y^j\}_{j \in [0,k]}} [H(\hat{p}(\cdot,k)|\{Y^j\})]$ for each accepting run.
\item
Select the trajectory with the minimum expected conditional entropy
\end{enumerate}
\end{algorithm}

The calculation of $E_{\{Y^j\}_{j \in [0,k]}} [H(\hat{p}(\cdot,k)|\{Y^j\})]$ from a given run proceeds as follows.  A run over $TS$ $q^0 \ldots q^k$ induces a sequence of reports $y^0,\ldots,y^k$.  We can find the estimate using \eqref{bayesFilter} that would result from observing a given sequence of reports and calculate $H(\hat{p}(\cdot,k)| \{y^j\}_{j \in [0,k]})$.  We can use the given run, the prior estimate pmf $\hat{p}(\cdot,0)$ and the likelihood function $f$ to construct  a pmf $p_{Y^0,\ldots,Y^k}(y^0,\ldots,y^k)$.  Taking these together we can calculate 

\begin{align}
\begin{split}
E_{\{Y^j\}_{j \in [0,k]}} [H(\hat{p}(\cdot,k)|\{Y^j\})]  & =  \\ \sum_{y^0,\ldots,y^k \in R_Y^k} H(\hat{p}(\cdot,k)| y^0,\ldots,y^k) & \times \\ p_{Y^0,\ldots,Y^k}(y^0,\ldots,y^k). \\
\end{split}
\end{align}

\subsection{Receding Horizon Control}
\label{recHoriz}

The exhaustive search (Algorithm \ref{brForce}) produces a solution that is optimal in expectation.  However, it is  computationally expensive (see Section \ref{comp}).   Algorithms exist to mitigate the computational costs incurred by Algorithm \ref{brForce}\cite{SinghJAIR09efficientinformative,MelioAAAAI07nstip}.

Algorithm \ref{brForce} gives a  non-reactive trajectory computed before the robot collects any additional information about $S$.   The optimal path is calculated based on the topology of $\mathcal{P}$, the sensor noise, and some initial guess $\hat{p}(\cdot,0)$.  
What if a sample path of $Y^k$ is atypical or $\hat{p}(\cdot,0)$ is a bad guess?  After making $l< k$ transitions, we cannot guarantee that 

$ \underset{\{q^{j}\}_{j \in [l,k]}} {\arg \min}  E_{\{Y^j\}} [  H(\hat{p}(\cdot,k)|\{Y^j\},\{y^m\}_{m \in [0,l)})]$ 

is the same as the end of the trajectory calculated using Algorithm \ref{brForce}. In the next section, we propose an on-line receding horizon algorithm that addresses the issues of computational explosion and non-reactivity.  
%The pre-computed trajectory is non-reactive to changes in the estimator.   

%The optimal decision trajectory over the finite horizon is calculated and the first decision in the trajectory is enacted.  After each transition, a new finite-horizon problem is constructed and solved and the next transition selected from the finite-horizon solution.  

%The authors of \cite{MelioAAAAI07nstip} add reactivity to the related informationally constrained shortest path problem by periodically performing a dynamic programming (DP) algorithm over possible paths to a fixed state.   
%We propose to use RHC instead of DP for our problem because the short-horizon problem only requires us to consider paths over a subset of $EP$ and because RHC provides the best path based on our most up-to-date information rather than requiring long periods between DP recalculations Further,  a DP algorithm is even more computationally expensive than the brute force solution as it involves both enumeration over trajectories and enumeration over sample paths  of $\Omega$.

\subsubsection{Algorithm description}

In the RHC approach to Problem \ref{conippp}, we select some horizon $b$ and at each time $l$ solve the following problem

\begin{subequations}
\label{optProb}
\begin{equation}
\label{mpc}
\begin{array}{c}
\underset{\{\delta_j \in \Delta_{\mathcal{P}}\}_{j \in [l,l+b]}} \min  E_{\{Y^j\}} [H(\hat{p}(\cdot,k)|\{Y^j\})]) \\
\text{subject to}  
\end{array} \end{equation} \begin{equation} \label{term}
\chi^{l+b}_{\text{opt}} \in N_{r,\chi_f}(\chi^l,b) \end{equation}  \begin{equation} \label{lyap}
W(\chi^{l+b}_{\text{opt}}) < W(\chi^{l+b-1}_{\text{pred}} ), \text{ if } N_{r,\chi_f}(\chi^l,b) \neq \chi_f  \end{equation} \begin{equation} \label{nonCyc}
\chi^{l+1}_{\text{opt}} \not\in \{\chi^{j}\}_{j \in [l_r,l]}, \text{ if } \chi^{l} \in N_r(\chi_f,b),
\end{equation}
\end{subequations} 

\noindent

where $\chi^l = (q^l,\sigma^l)$ is the state of $\mathcal{P}$ at time $l$, $\chi_{\text{opt}}$ is a state in the optimal  finite-horizon trajectory calculated at time $l$, $\chi_{\text{pred}}$ is a state in the optimal finite-horizon trajectory previously calculated at time $l-1$, $N_r(\chi,n)$ is the neighborhood of states about $\chi$ that are reachable in $n$ or fewer transitions, and $l_r$ is the minimum value of $l$ such that $\chi^l \in N_r(\chi_f,b)$. The optimization \eqref{optProb} is solved in the same manner as Algorithm \ref{brForce} in which feasible paths on $\mathcal{P}$ over the short horizon are enumerated, projected back to $TS$, and their expected impact on conditional entropy evaluated.  
The function $W: Q \times \Sigma \to \mathbb{R}$ is defined as

\begin{equation}
\label{lyapFun}
\begin{array}{c}
\chi_0 = (q_0, \sigma^{0}) \\
\chi_f  = \underset{\chi_k \in Q \times F} {\arg \max} \text{ }  D(\chi_0,\chi_k) \text{ s. t. } D(\chi_0,\chi_f) < \infty \\
W(\chi_j) =   D(\chi_j,\chi_f), \\ 
\end{array}
\end{equation}

\noindent

where $D(\cdot,\cdot)$ is the shortest graph distance between two states in $\mathcal{P}$.  The distance between two adjacent states is given by the weighting $d'$.  $N_{r,\chi_f}(\chi,n)$  is the constrained $n$-step reachability neighborhood 

\begin{equation}
N_{r,\chi_f}(\chi,n) = \left \{ \begin{array}{ll}  N_r(\chi,n) & \chi_f \not\in N_r(\chi,n) \\ \chi_f & \chi_f \in N_r(\chi,n) \end{array} \right. .
\end{equation}

\noindent

The extra conditions \eqref{term}-\eqref{nonCyc} ensure convergence to $\chi_f$ in finite time.  Constraint \eqref{term} ensures that if $\chi_f$ is reachable from the current position, the terminal state in the finite-horizon trajectory is $\chi_f$. Constraint  \eqref{lyap} is similar to a decreasing energy constraint used in Lyapunov convergence analysis.  It ensures that the finite-horizon trajectory moves closer to an accepting trajectory as time increases.  Condition \eqref{nonCyc} ensures that $\mathcal{P}$ does not cycle infinitely between non-accepting states.

We construct a receding horizon algorithm  adapted from \cite{DiLaBe-ACC-2012} to solve Problem \ref{conippp}. 

\begin{algorithm}[Receding Horizon]\
\label{recHorizAlg}

l = 0 

$\chi = (q_0,\sigma^0)$

\textbf{While} ($\chi \neq \chi_f$)

		\tab $\{\chi^m_{\text{pred}}\}_{m \in [l,l+b-1]} = \{\chi^m_{\text{opt}}\}_{m \in [l,l+b-1]}$
	
		\tab $\{\chi^m_{\text{opt}}\}_{m \in [l+1,l+b]}$ = \em solution to \em  \eqref{optProb}
	
		\tab $\chi = \chi_{\text{opt}}^{l+1}$
		
		\tab $l ++$
	
\end{algorithm}

If at least one satisfying run exists on $\mathcal{P}$ (i.e. if $W(\chi_0)$ is finite), then any path produced by Algorithm \ref{recHorizAlg} satisfies the specification $\phi$.  This is formalized in the following theorem.  The proof proceeds in a similar manner as in \cite{DiLaBe-ACC-2012}.  
 
\begin{thm}[scLTL satisfaction]
\label{convergeThm}
If $W(\chi_0) < \infty$, applying Algorithm \ref{recHorizAlg} to Problem \ref{conippp} will result in an accepting run on $\mathcal{P}$.
\end{thm}

\begin{proof} \noindent

This proof uses the following properties of $W(\cdot)$

\begin{enumerate}
\item
$W(\chi_j) = 0 \Leftrightarrow \chi_j = \chi_f$
\item 
$W(\chi_j) = \infty \Leftrightarrow \chi_f$ not reachable from $\chi_j$
\item 
$W(\chi_j) < \infty \Rightarrow  \exists \chi_k \in N_r(\chi_j,1)$ such that $W(\chi_k) < W(\chi_j)$
\end{enumerate}

First, we must prove that if $W(\chi_0) < \infty$, \eqref{mpc} has a solution for every time $l$.  If $\chi^l$ is such that $N_{r,\chi_f}(\chi^l,b) = \chi_f$, then the condition $W(\chi_0) < \infty$ and \eqref{lyap} together imply  $W(\chi^l) < \infty$.  Thus, by definition of $N_{r,\chi_f}(\chi^l,b)$ and \eqref{term}, there exists a trajectory of $b$ or fewer steps from $\chi^l$ to $\chi_f$ that does not include any previously selected states.    If, on the other hand, $N_{r,\chi_f}(\chi^l,b) = N_r(\chi^l,b)$,  consider the $b-1$ step trajectory given by $T_{b-1} = \chi^l \chi^{l+1}_{\text{pred}}  \ldots \chi^{l+b-1}_{\text{pred}}$.  Since this trajectory is in $N_r(\chi^{l-1},b)$ it must also be in $N_r(\chi^{l},b-1)$.  By the third property of $W(\cdot)$, there is a neighbor $\chi_k$ of $\chi^{l+b-1}_{\text{pred}}$  such that $W(\chi_k) < W(\chi^{l+b-1}_{\text{pred}})$.  Define a $b$-step trajectory $T_{b}= T_{b-1}\chi_k$.  $T_{b}$ definitely exists and satisfies \eqref{term}-\eqref{nonCyc}.

Given that a solution exists for every time $l$, we must prove that application of Algorithm \ref{recHorizAlg} will drive $\mathcal{P}$ to the state $\chi_f$ in finite time.  Define the function $\alpha(l) = W(\chi^{l+b}_{\text{opt}})$.  The constraint \eqref{lyap} forces $\alpha$ to be monotonically decreasing and  by definition of $W$,  $\alpha$ has a global minimum value of 0.  $\mathcal{P}$ has a finite number of states, so there exists some finite time $l^*$ such that $\alpha(l^*) = 0$.  After $l^*$, all future states are constrained to $N_r(\chi_f,b)$.

Once $\mathcal{P}$ is in $N_r(\chi_f,b)$, \eqref{lyap} is insufficient to guarantee convergence to $\chi_f$.  
However, imposing the non-repetition constraint  \eqref{nonCyc} forces $\mathcal{P}$ to converge to $\chi_f$ in at most $|N_r(\chi_f,b)|$ steps after entering $N_r(\chi_f,b)$.

\end{proof}

Note that intuitively in an environment with spatially distributed information we expect that longer paths  generally will be more informative.  It may seem that using conditions  \eqref{term}-\eqref{nonCyc} to ensure convergence causes Algorithm \ref{recHorizAlg} to converge more quickly (produce shorter paths) than is desirable.  This effect is offset by the reactivity of the algorithm and selection of the optimal local trajectories at each time step.  Our approach can be adapted to further address path length concerns by including minimum path length constraints  in Algorithm \ref{recHorizAlg} or by specifying in the scLTL constraint $\phi$ a set of spatially dispersed regions that the robot must visit.

\subsubsection{Computational complexity}
\label{comp}

Define $K(\chi_0,\chi_f,t)$ as the number of simple paths of length less than or equal to $t$ that connect $\chi_0$ to $\chi_f$ in $\mathcal{P}$. Let $t^*$ be the length of the longest simple path in $\mathcal{P}$ and let  $\kappa (\mathcal{P},t) = \underset{\chi_0,\chi_f \in (Q \times \Sigma)^2 } \max K(\chi_0,\chi_f,t)$.  Calculating the expected impact of each transition on the conditional entropy of the estimate requires $|R_S|$ calculations.  
%Let $M = \max_{\chi \in Q \times \Sigma} |N_o(\chi)|$.  
The computational complexity of Algorithm \ref{brForce} is  therefore $O(|R_S| t^* \kappa(\mathcal{P},t^*))$.
 
For Algorithm \ref{recHorizAlg}, consider a single solution of \eqref{mpc} with short horizon $b$. The number of possible paths is bounded by $\ell \kappa(\mathcal{P},b)$, where $\ell$ is the number of edges of the maximally connected state in $\mathcal{P}$. The complexity of a single solution to \eqref{optProb} is $O(|R_S| b\ell \kappa(\mathcal{P},b))$.  Constraints \eqref{term}-\eqref{nonCyc} mean that \eqref{optProb} is solved at most $N=|Q \times \Sigma|$ times.  The complexity of Algorithm \ref{recHorizAlg} is $O(|R_S| b\ell N \kappa(\mathcal{P},b))$.

A comparison of worst-case complexity depends on the size and topology of $\mathcal{P}$.  Note that for a product automaton of large size and high enough connectivity, the function $\kappa(\mathcal{P},\cdot)$ increases exponentially in path length.  For such systems, Algorithm \ref{recHorizAlg} has the lowest worst-case complexity.  While it may seem disingenuous to compare the complexity of an off-line algorithm against an on-line algorithm, note that as the size and connectivity of $\mathcal{P}$ grow, it becomes infeasible to solve Algorithm \ref{brForce} in a reasonable amount of pre-mission time before it becomes infeasible to calculate Algorithm \ref{recHorizAlg} on-line. 

\section{Simulation Study}
\label{caseStudy}
We performed a simulation study demonstrating the use of Algorithm \ref{recHorizAlg} to solve Problem \ref{conippp}.   We assumed the transition system is the quotient of a gridded partition and that all neighboring regions are deterministically reachable.  Our variable of interest is $S = [S_j]_{j:q_j \in Q}$ where $R_{S_j} = \{0,1\}$.  
We assume that the $S_j$ are mutually independent.  After a transition to a new region, the robot  returns a report $y^k \in \{0,1\}$.   Our estimate is formed using a prior pmf and the Bayesian filter \eqref{bayesFilter}.  

We assume that the volume of a region $q \in Q$  is sufficiently small compared to the volume observable by the robot's sensors such that the robot will receive information from adjacent regions.  We model this overlap by a set $E_{\text{meas}}$, where the existence of an element $e_{jk} \in E_{\text{meas}}$  indicates information from region $q_j$ can be gathered while the robot is in $q_k$. We assume here that $E_{\text{meas}} = \{e_{jk}| (q_j,u,q_k) \in Trans \}$, though this assumption does not need to hold in general. 
Each  element of $E_{\text{meas}}$ is weighted according to the distance $d_M(q_j,q_k)$ that represents the amount of information contained in region $q_k$ that can be observed from region $q_j$.   Define the observation neighborhood of $q_k$ as $N_o(q_k) = \{q_i | e_{ik} \in E_{\text{meas}} \} \cup q_k$.  
We assume independent correct report rates $\mu(q^k,q_j) = \text{Pr}(Y^k = 1 | $ agent at $q^k, s_j = 1)$ and a constant false alarm rate $r$ such that the overall alert likelihood is 

\begin{align}
f(1,s,q^k) =  \left \{
\begin{array}{c}
r \text{  if  } s_j = 0 \text{  } \forall j: q_j \in N_o(q^k) \\
1- \underset{q_j \in N_o(q^k)}\prod 1 - \mu(q^k,q_j)s_j, \text{  else}
\end{array}
\right .
\end{align}

Since our reports are binary, we can calculate $f(0,s,q^k)$ from $f(1,s,q^k)$.  Our detection model is given by 

\begin{equation}
\mu(q^k,q_j) =\mu_0 e^{- \lambda d_M(q^k,q_j)}
\end{equation}

The propositions in our scenario are $AP = \{\text{D1},\text{D2},\text{C},\text{U} \}$.  The specification we wish to satisfy is ``Visit D1 before visiting D2 and visit D2 before ending in C while avoiding U".  The task is formalized as 

 \begin{equation}
\label{naturalDisaster}
\begin{array}{c}
\underset{\{q^j\}_{j\in [0,k]}}  \min  E_{\{Y^j\}} [H(\hat{p}(\cdot,k)|\{Y^j\})] \\
\text{subject to} \\
(\neg \text{U} \LTLUNTIL \text{C} ) \wedge (\neg \text{C} \LTLUNTIL \text{D2} ) \wedge (\neg \text{D2} \LTLUNTIL \text{D1} )
\end{array}
\end{equation} 

We generated a 5 x 5  grid-like abstraction with fixed initial state and terminal `C' state.  Our simulation was constructed using NetworkX graph algorithms \cite{Hagberg08networkx} and the model-checking algorithms of scheck \cite{LatvalaBook03scheck}.  We performed 100 Monte Carlo trials with randomly placed `D1', `D2', and `U' labels.  The sensing parameters were $\mu_0 = 0.9$, $r =0.01$, and $\lambda = 0.01$.  Weightings $d_M$ between adjacent states were drawn according to uniform distributions over $(0,10)$, graph distances between two adjacent states were set at a value of 1, and the $s_j$ were generated according to a Bernoulli distribution with parameter $p = 0.08$.  Figure \ref{samplePaths} shows sample paths resulting from using Algorithm \ref{recHorizAlg} to solve \eqref{naturalDisaster}.  Each run satisfied the given constraint specification. The average terminal entropy $H(\hat{p}(\cdot,k)|y^0,\ldots,y^k)$ over all trials was 14.78 bits.  The average CPU time required for construction of $\
mathcal{P}$ and optimal path finding per trial was 2.61 s on a machine with a 2.66 GHz Intel Core 2 Duo processor and 4 GB of memory.

\begin{figure*}
\begin{center}
\begin{tabular}{c c c}
% \includegraphics[scale=0.225]{../../../Simulation/MinimumEntropyTestbed/pathSamples/size_25_iteration_78} &
% \includegraphics[scale=0.225]{../../../Simulation/MinimumEntropyTestbed/pathSamples/size_25_iteration_11} &
% \includegraphics[scale=0.225]{../../../Simulation/MinimumEntropyTestbed/pathSamples/size_25_iteration_43}  \\
% (a) & (b) & (c) \\
\includegraphics[scale=0.225]{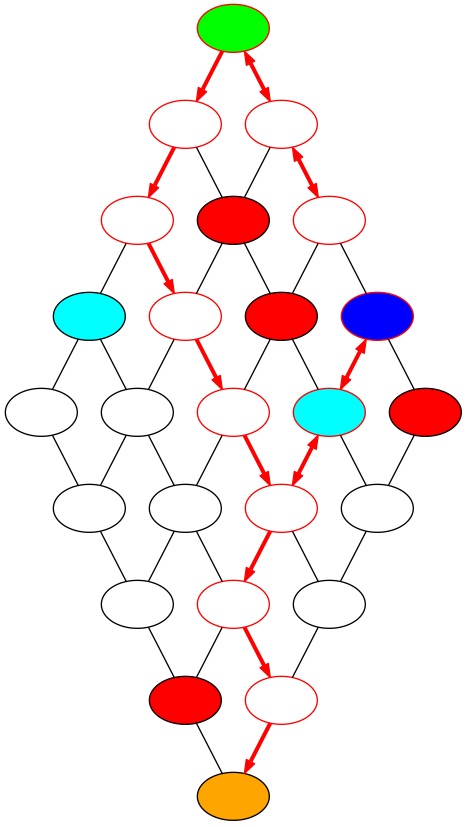} &
\includegraphics[scale=0.225]{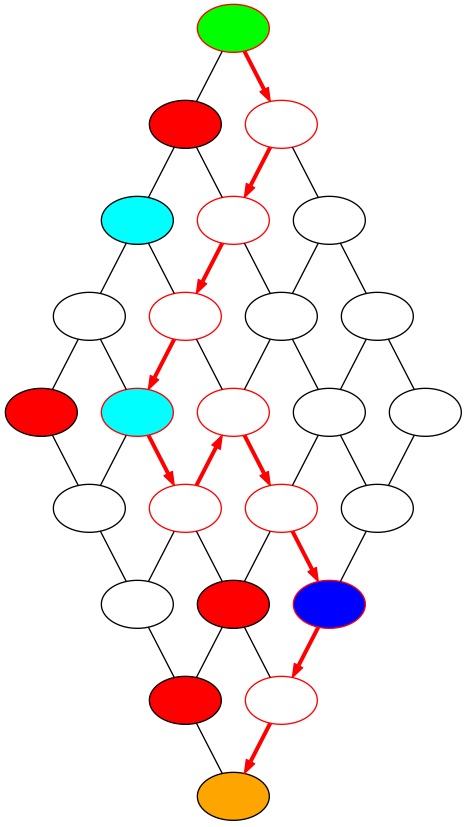} &
\includegraphics[scale=0.225]{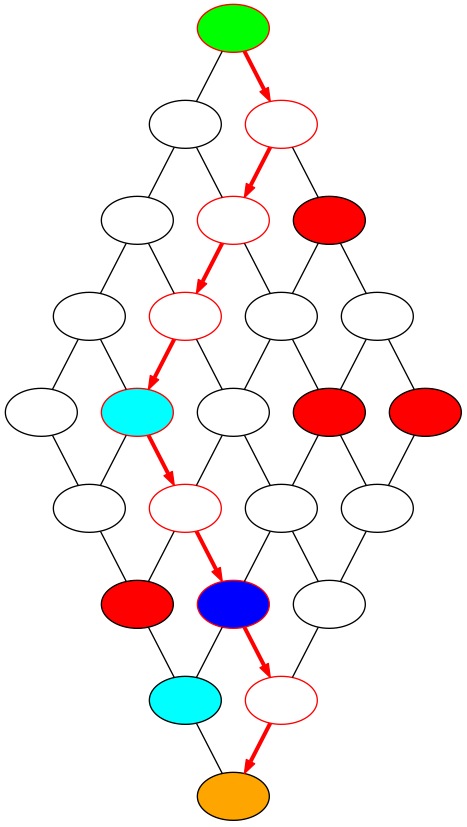}  \\
(a) & (b) & (c) \\
\end{tabular}
\caption{\label{samplePaths}  Three sample paths produced by our receding horizon algorithm to solve the scLTL-constrained informative path planning problem.  The red transitions indicate the path followed by the robot with arrows indicating direction.  The specification is ``Beginning at green, visit a light blue region, and then visit a dark blue region,  and finally visit the orange region while always avoiding red regions".   This corresponds to the formula in \eqref{naturalDisaster} where the orange state is `C', red states are `U', light blue are `D1', and dark blue are `D2'.}
\end{center}
\end{figure*}

In order to compare the performance of Algorithms \ref{brForce} and \ref{recHorizAlg}, we solved \eqref{naturalDisaster} over a single 6 x 6 transition system generated in the same manner as above. 
We chose the larger system in order to make comparisons over a larger number of accepting runs.  
We used Algorithm \ref{brForce} to find the optimal path in the constrained environment and constructed the pmf of the terminal entropy.  We also performed 250 Monte Carlo trials using Algorithm \ref{recHorizAlg} over the same environment and constructed the empirical pmf of the resulting terminal entropies.  Histogram representations of the two pmfs are shown in Figure \ref{compHists}.  The mean, median, and variance  are 26.30 bits, 26.44 bits, and 3.67 $\text{bits}^2$, respectively, for the pmf from Algorithm \ref{brForce} and 25.86 bits, 26.44 bits, and 2.80 $\text{bits}^2$, respectively, for the empirical pmf from Algorithm \ref{recHorizAlg}.  These results confirm our intuition about reactivity and performance: Algorithm \ref{recHorizAlg} performs better in expectation and has lower performance variability than Algorithm \ref{brForce}.

\begin{figure}
\begin{center}
\begin{tabular}{c}
\includegraphics[scale=0.4]{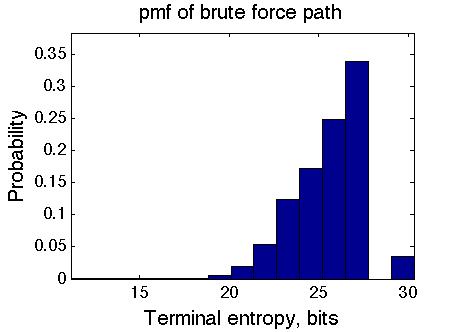} \\ (a) \\ 
\includegraphics[scale=0.4]{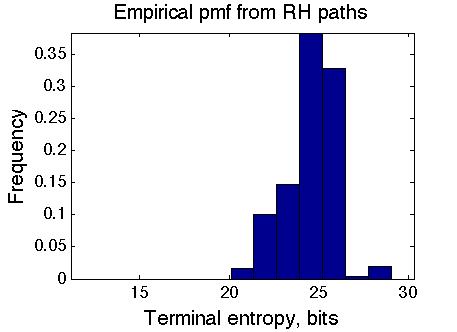} \\ (b) \\
\end{tabular}
\caption{\label{compHists} Histograms of (a) the pmf of the terminal entropy found when following the path from Algorithm \ref{brForce} and (b) the empirical pmf of the terminal entropy that resulted when the paths were calculated using Algorithm \ref{recHorizAlg}.  These histograms show that the mean and variance of the pmf of the terminal entropy is lower for the paths generated by Algorithm \ref{recHorizAlg} than the for the path generated by Algorithm \ref{brForce}.   The lower mean indicates that using Algorithm \ref{recHorizAlg} will result in a lower entropy estimate on average. The lower variability means that we are less likely to have a high entropy estimate when using Algorithm \ref{recHorizAlg}.  Algorithm \ref{brForce} took 1741 s of CPU time to complete and Algorithm \ref{recHorizAlg} took an average of 2.94 s of CPU time per execution to complete.}
\end{center}
\end{figure}

\section{Conclusions and Future Work}
\label{goals}

In this paper, we considered planning an informative path for a
robotic agent subject to temporal logic specifications.  We modeled
the robot as moving deterministically on a graph with noisy sensor measurements at each
node.  We proposed a receding horizon algorithm for solving this
problem in an on-line, computationally efficient manner while still ensuring specification satisfaction.  We compared the performance of our algorithm
with an off-line exhaustive search method in a simulation study.  Our algorithm
out performed the exhaustive search method, producing lower entropy
estimates with less computational overhead.  % the robot motion,
% sensing, and environment with a suitable model.  model suitable for
% tcreated an abstraction of a robotic agent that can effectively model
% both temporal logic properties and estimation of an unkown variable
% and defined the scLTL-constrained path planning problem over this
% model.  We showed how to combine the tools of information theory and
% model checking to solve this problem using exhaustive search.  We
% improved upon this solution by introducing a receding horizon
% algorithm that is guaranteed to produce feasible trajectories in a
% more efficient manner.  We confirmed these results and showed in
% simulation that our algorithm on average produces lower-entropy
% estimates than the exhaustive search solution.

One natural extension to this work is to plan a path that optimizes
some other quantity (e.g. path length or graph distance) subject to a
minimum level of mutual information. That is, we make the information
content of the path a constraint rather than an
objective.  %These problems can likely be solved in a straightforward manner using robust optimization techniques. \cite{robustOpt}
Another possible extension is to consider cases in which the
satisfaction of the temporal logic specification relies on some unknown
quantity.  Consider, for instance, a rescue robotics scenario in which
the robot is tasked not only with finding survivors, but also moving
the survivors to a medical station.  In this case, planning a path to
the medical station is impossible until the robot knows the survivor's
location.  This extension would allow us to use formal synthesis
methods in a more reactive manner.  More generally, we expect that the
fusion of information theoretic tools with formal control synthesis
will yield robotic control policies that are reactive to noisy,
real-world environments while still providing provably correct
performance.
%Feasible solution requires on-line estimation of the unknown required quantities.  
%If we incorporate these estimates into the state of the system, we can use the estimator dynamics to form a partially observable markov decision process (POMDP) \cite{pomdpltl,pomdp} .  Optimization problems over POMDPs are likely well-suited to solution via receding horizon methods.

%Finally, we can extend our approach to model constrained motion planning in distributed systems.  In distributed systems, agents move and gather information independently to perform some common task.  Our work here and the other proposed extensions involve making decisions on-line that are at least partially predicated by changes in the estimator.  In a distributed system, if an agent transmits its estimate to a neighbor, the data fusion process enacted by the neighboring agent drastically changes its own estimate and motion plan.  This fusion process adds a layer of decision-making as individual agents have to decide not only which transitions to make but also decide when and with whom to share gained information.  Integrating fusion dynamics into our model of the distributed system would allow us to incorporate message passing into our decision-making policy without communication entering explicitly in our optimization.

\addtolength{\textheight}{-2.5cm}

\bibliographystyle{plain}
\bibliography{mpcltl}

\end{document}